\documentclass{llncs}
\usepackage{amsmath}
\usepackage{times}
\usepackage{indentfirst}
\usepackage{graphicx}
\usepackage{amssymb}

\newtheorem{faulse assertion}{faulse assertion}
\newtheorem{counter-example}{counter-example}

\begin{document}

\title{Applications of repeat degree on coverings of neighborhoods}         

\author{Hua Yao, William Zhu\thanks{Corresponding author.
E-mail: williamfengzhu@gmail.com(William Zhu)} }
\institute{Lab of Granular Computing,\\
Zhangzhou Normal University, Zhangzhou, China}



\date{\today}          
\maketitle

\begin{abstract}
In covering based rough sets, the neighborhood of an element is the intersection of all the covering blocks containing the element.
All the neighborhoods form a new covering called a covering of neighborhoods.
In the course of studying under what condition a covering of neighborhoods is a partition, the concept of repeat degree is proposed, with the help of which the issue is addressed. This paper studies further the application of repeat degree on coverings of neighborhoods. First, we investigate under what condition a covering of neighborhoods is the reduct of the covering inducing it. As a preparation for addressing this issue, we give a necessary and sufficient condition for a subset of a set family to be the reduct of the set family. Then we study under what condition two coverings induce a same relation and a same covering of neighborhoods. Finally, we give the method of calculating the covering according to repeat degree.
\end{abstract}

\section{Introduction}
Rough set theory is first proposed by Pawlak~\cite{Pawlak82Rough,Pawlak91Rough} for dealing with vagueness and granularity in information systems. In theory, rough sets have been connected with matroids~\cite{TangSheZhu12matroidal,WangZhuZhuMin12matroidalstructure}, lattices~\cite{Dai05Logic,EstajiHooshmandaslDavvaz12Roughappliedtolattice,Liu08Generalized,WangZhu13Quantitative}, hyperstructure theory~\cite{YamakKazanciDavvaz11Softhyperstructure},
topology~\cite{Kondo05OnTheStructure,LashinKozaeKhadraMedhat05Rough,Zhu07Topological}, fuzzy sets~\cite{KazanciYamakDavvaz08TheLower,WuLeungMi05OnCharacterizations}, and so on.
Rough set theory is built on equivalence relations or partitions.
But equivalence relations and partitions are too restrictive for many applications.
To address this issue, several meaningful extensions of Pawlak rough sets have been proposed.
Among them, Zakowski~\cite{Zakowski83Approximations} has used coverings to establish covering based rough set theory.
Many scholars~\cite{BonikowskiBryniarskiWybraniecSkardowska98Extensions,Bryniarski89ACalculus,ChenDeganZhangWenxiu2006Roughapproximations,Pomykala87Approximation,Pomykala88Ondefinability,ZhuWang03Reduction} have done deep researches on this theory.
Recently, covering based rough set theory gained some new development~\cite{MuratDiker2012Textures,TianYangQingguoLiBileiZhou2013Relatedfamily,YiyuYaoBingxueYao2012Coveringbased,YanlanZhangMaokangLuo2013Relationshipsbetweencovering}.

In covering based rough sets, the neighborhood of an element is the intersection
of all the covering blocks containing the element. All the neighborhoods
form a new covering called a covering of neighborhoods.
Among various types of covering based rough sets, there are some~\cite{SamantaChakraborty09Covering,QinGaoPei07OnCovering,Yao98Relational,WeihuaXuWenxiuZhang2007Measuringroughness,zhu2011covering} defined by neighborhoods. Furthermore, there are many properties~\cite{HuangZhu12Geometriclattice,LiuSai09AComparison,Ma12Onsometypes,Zhu09RelationshipBetween} of covering based rough sets associated with the properties of coverings of neighborhoods.
This makes coverings of neighborhoods be important research subject.
Lin~\cite{Lin88Neighborhoodsystems} augmented the relational database with neighborhoods.
Yao~\cite{Yao98Relational} presented a framework for the formulation, interpretation, and comparison of a specific class of
neighborhood systems (called 1-neighborhood systems) induced by binary relations and rough
set approximations.
By means of consistent functions based on neighborhoods, Wang et al.~\cite{WangChenSunHu12Communication} dealt with the reduction issues on covering decision systems. Many scholars~\cite{chenstructure,FanHuXiaoZhang12Study,QinGaoPei07OnCovering,YunGeBai11Axiomatization} and we studied under what condition a covering of neighborhoods is a partition. In the course of studying this issue, we proposed the concept of repeat degree. With the help of this concept, we addressed this issue as well as the issue that under what condition a covering of neighborhoods is equal to the covering inducing it.

In this paper, we study further the application of repeat degree on coverings of neighborhoods. First, we investigate under what condition a covering of neighborhoods is the reduct of the covering inducing it. As a preparation for addressing this issue, we give a necessary and sufficient condition for a subset of a set family to be the reduct of the set family. Then we study under what condition two coverings induce a same relation and a same covering of neighborhoods. We prove these two issues are equivalent. Finally, we give the method of calculating the covering according to repeat degree and prove that partial information of repeat degree cannot determine the covering.

The remainder of this paper is organized as follows.
In Section~\ref{S:Preliminaries}, we review the relevant concepts.
In Section~\ref{Repeat degree and its properties}, we introduce the concept of repeat degree and study some properties of it.
In Section~\ref{A condition for a covering of neighborhoods to be a reduct}, we first study under what condition a subset of a set family is a reduct of the set family. Then we present a sufficient and necessary condition for a covering of neighborhoods to be the reduct of the covering inducing it.
In Section~\ref{A condition for two coverings to induce a same relation and a same covering of neighborhoods}, we present a sufficient and necessary condition for two coverings to induce a same relation and a same covering of neighborhoods.
In Section~\ref{Calculating the covering by repeat degree}, we give the method of calculating the covering according to repeat degree.
Section~\ref{S:Conclusions} concludes this paper.

\section{Preliminaries}
\label{S:Preliminaries}

For a better understanding to this paper, in this section, some basic concepts are introduced. In this paper, we denote $\cup_{X\in S}X$ by $\cup S$, where $S$ is a set family.

\begin{definition}(Covering)
\label{definition1}
Let $U$ be a universe of discourse and $\mathbf{C}$ be a family of subsets of $U$.
If $\emptyset\notin\mathbf{C}$ and $\cup \mathbf{C}=U$, $\mathbf{C}$ is called a covering of $U$.
Every element of $\mathbf{C}$ is called a covering block.
\end{definition}

In the following discussion, unless stated to the contrary, the universe of discourse $U$ is considered to be
finite and nonempty.
Neighborhood~\cite{BonikowskiBryniarskiWybraniecSkardowska98Extensions,Lin88Neighborhoodsystems,Pomykala88Ondefinability,QinGaoPei07OnCovering,WangChenSunHu12Communication,Yao98Relational} is a concept used widely in covering based rough sets.
It is defined as follows.

\begin{definition}(Neighborhood)
\label{definition2}
Let $\mathbf{C}$ be a covering of $U$.
For any $x\in U$, the neighborhood of $x$ is defined by: $N_{\mathbf{C}}(x)=\cap\{K\in\mathbf{C}:x\in K\}$.
When there is no confusion, we omit the subscript $\mathbf{C}$.
\end{definition}

It is obvious $x\in N_{\mathbf{C}}(x)$ and for any $x\in K\in\mathbf{C}$, $N_{\mathbf{C}}(x)\subseteq K$.
The following proposition gives an important property of neighborhoods.

\begin{proposition}~\cite{WangChenSunHu12Communication}
\label{proposition3}
Let $\mathbf{C}$ be a covering of $U$.
For any $x,y\in U$, if $y\in N(x)$, $N(y)\subseteq N(x)$.
\end{proposition}

If $y\in N(x)$ and $x\in N(y)$, by the above proposition, we have $N(x)=N(y)$.
All the neighborhoods induced by a covering of a universe form a set family. This set family is still a covering of the universe.
This type of set families have been studied by many scholars~\cite{chenstructure,FanHuXiaoZhang12Study,QinGaoPei07OnCovering,WangChenSunHu12Communication,YunGeBai11Axiomatization}. However, both the term and the mark of it are not identical. In this paper, we call it covering of neighborhoods and cite the mark proposed by Wang et al.~\cite{WangChenSunHu12Communication}.

\begin{definition}(Covering of neighborhoods)
\label{definition4}
Let $\mathbf{C}$ be a covering of $U$. The covering of neighborhoods induced by $\mathbf{C}$ is defined by:
$Cov(\mathbf{C})=\{N(x):x\in U\}$.
\end{definition}

According to $x\in N(x)$, it is obvious $\cup Cov(\mathbf{C})=\cup\mathbf{C}$.

\section{Repeat degree and its properties}
\label{Repeat degree and its properties}

In this section, we propose a concept called repeat degree and study the properties of it. Particularly, a relationship between it and neighborhoods is presented. In the following discussion, unless stated to the contrary, for any set family $\mathbf{C}$, $\cup\mathbf{C}$ is considered to be
finite and nonempty.

\begin{definition}(Repeat degree)
\label{definition5}
Let $\mathbf{C}$ be a covering on $\cup\mathbf{C}$ and $X\subseteq\cup\mathbf{C}$. The repeat degree of $X$ with respect to covering $\mathbf{C}$ is defined by: $\partial_{\mathbf{C}}(X)=|\{K\in\mathbf{C}:X\subseteq K\}|$. When there is no confusion, we omit the subscript $\mathbf{C}$.
\end{definition}

For the convenience of writing, we denote $\partial_{\mathbf{C}}(\{x\})$ as $\partial_{\mathbf{C}}(x)$.
According to the above definition, for any $X\subseteq Y\subseteq\cup\mathbf{C}$, it follows that $\partial_{\mathbf{C}}(X)\geq\partial_{\mathbf{C}}(Y)$.
To illustrate the concept of repeat degree, let us see the following example.

\begin{example}
\label{example6}
Let $U=\{1,2,3,4\}$ and $\mathbf{C}=\{\{1,2\},\{2,3,4\},\{3,4\}\}$.
Then $\partial(\emptyset)=3$, $\partial(1)=1$, $\partial(2)=\partial(3)=\partial(4)=2$,
$\partial(\{1,2\})=\partial(\{2,3\})=\partial(\{2,4\})=1$, $\partial(\{1,3\})=\partial(\{1,4\})=0$, $\partial(\{3,4\})=2$, $\partial(\{2,3,4\})=1$, $\partial(\{1,3,4\})=\partial(\{1,2,4\})=\partial(\{1,2,3\})=0$ and $\partial(\{1,2,3,4\})=0$.
\end{example}

Repeat degree has the following basic property.

\begin{proposition}
\label{proposition7}
Let $\mathbf{C}$ be a covering on $\cup\mathbf{C}$.
For any $x,y\in\cup\mathbf{C}$,
$\partial(x)=\partial(\{x,y\})\Leftrightarrow\{K\in\mathbf{C}:\{x\}\subseteq K\}=\{K\in\mathbf{C}:\{x,y\}\subseteq K\}$.
\end{proposition}

\begin{proof}

$(\Rightarrow)$: It is obvious $\{K\in\mathbf{C}:\{x,y\}\subseteq K\}\subseteq\{K\in\mathbf{C}:\{x\}\subseteq K\}$.
If $\{K\in\mathbf{C}:\{x,y\}\subseteq K\}\neq\{K\in\mathbf{C}:\{x\}\subseteq K\}$, $\{K\in\mathbf{C}:\{x,y\}\subseteq K\}$ is a proper subset of $\{K\in\mathbf{C}:\{x\}\subseteq K\}$.
Since $\{K\in\mathbf{C}:\{x\}\subseteq K\}$ is a finite set, $|\{K\in\mathbf{C}:\{x,y\}\subseteq K\}|<|\{K\in\mathbf{C}:\{x\}\subseteq K\}|$.
Thus $\partial(\{x,y\})<\partial(x)$.
This is a contradiction to that $\partial(x)=\partial(\{x,y\})$.

$(\Leftarrow)$: It is straightforward. $\Box$
\end{proof}

According to the above proposition, we obtain an important relationship between repeat degree and neighborhoods.

\begin{proposition}
\label{proposition8}
Let $\mathbf{C}$ be a covering on $\cup\mathbf{C}$.
For any $x,y\in\cup\mathbf{C}$, $y\in N(x)$ iff $\partial(x)=\partial(\{x,y\})$.
\end{proposition}

\begin{proof}
According to Proposition~\ref{proposition7}, we have\\
$y\in N(x)\Leftrightarrow\forall K((K\in\mathbf{C}\wedge x\in K)\rightarrow(y\in K))\Leftrightarrow\forall K((K\in\mathbf{C}\wedge x\in K)\rightarrow(K\in\mathbf{C}\wedge \{x,y\}\subseteq K))\Leftrightarrow\forall K((K\in\mathbf{C}\wedge x\in K)\leftrightarrow(K\in\mathbf{C}\wedge \{x,y\}\subseteq K))\Leftrightarrow\{K\in\mathbf{C}:\{x\}\subseteq K\}=\{K\in\mathbf{C}:\{x,y\}\subseteq K\}\Leftrightarrow\partial(x)=\partial(\{x,y\})$. $\Box$
\end{proof}

\section{A condition for a covering of neighborhoods to be a reduct}
\label{A condition for a covering of neighborhoods to be a reduct}

In 2003, Zhu et al.~\cite{ZhuWang03Reduction} proposed two concepts called reducible element and the reduct of a covering, which have important applications in covering based rough set theory. In this section, we will discuss under what condition a covering of neighborhoods is a reduct. First, we propose a new mark.

\begin{definition}
\label{definition23}
Let $\mathbf{C}$ be a covering on $\cup\mathbf{C}$. We define $I(\mathbf{C})=\{\cup D:D\subseteq \mathbf{C}\}$.
\end{definition}

$I(\mathbf{C})$ has the following simple property.

\begin{proposition}
\label{proposition24}
Let $\mathbf{C}$ be a covering on $\cup\mathbf{C}$ and $B\subseteq \mathbf{C}$. Then $I(B)\subseteq I(\mathbf{C})$.
\end{proposition}

\begin{proof}
For any $K\in I(B)$, we know that there exists some $D\subseteq B$ such that $K=\cup D$. It is obvious $D\subseteq \mathbf{C}$. Thus $K\in I(\mathbf{C})$, therefore $I(B)\subseteq I(\mathbf{C})$.
$\Box$
\end{proof}

Based on Definition~\ref{definition23}, we introduce the concept of reducible element, which is somewhat different in form from its definition in ~\cite{ZhuWang03Reduction}.

\begin{definition}
\label{definition9}
Let $\mathbf{C}$ be a covering on $\cup\mathbf{C}$. The reducible element family of $\mathbf{C}$ is defined by: $S(\mathbf{C})=\{K\in\mathbf{C}:K\in I(\mathbf{C}-\{K\})\}$. If $K\in S(\mathbf{C})$, $K$ is called a reducible element of $\mathbf{C}$, otherwise $K$ is called an irreducible element of $\mathbf{C}$.
\end{definition}

The following proposition presents a simple property of $S(\mathbf{C})$.

\begin{proposition}
\label{proposition36}
Let $\mathbf{C}$ be a covering on $\cup\mathbf{C}$ and $B\subseteq\mathbf{C}$. Then $S(B)\subseteq S(\mathbf{C})$.
\end{proposition}

\begin{proof}

For any $A\in S(B)$, by Definition~\ref{definition9}, we know that $A\in B\wedge A\in I(B-\{A\})$. By $B\subseteq\mathbf{C}$ and Proposition~\ref{proposition24}, we have that $A\in\mathbf{C}\wedge A\in I(\mathbf{C}-\{A\})$. Hence $A\in S(\mathbf{C})$. Therefore $S(B)\subseteq S(\mathbf{C})$.
$\Box$

\end{proof}

The following proposition indicates that deleting a reducible
element in a covering will not make any
original reducible element become an irreducible element of the new covering.

\begin{proposition}~\cite{ZhuWang03Reduction}
\label{proposition22}
Let $\mathbf{C}$ be a covering on $\cup\mathbf{C}$ and $K_{1}\in S(\mathbf{C})$. $K\in S(\mathbf{C})-\{K_{1}\}$ iff $K\in S(\mathbf{C}-\{K_{1}\})$.
\end{proposition}

For the convenience of application, we extend the above proposition.

\begin{proposition}
\label{proposition35}
Let $\mathbf{C}$ be a covering on $\cup\mathbf{C}$ and $F\subseteq S(\mathbf{C})$. $K\in S(\mathbf{C})-F$ iff $K\in S(\mathbf{C}-F)$.
\end{proposition}

\begin{proof}

$(\Rightarrow)$:
Let $|F|=n$.
We prove this proposition using induction on $n$.
If $n=1$, this proposition follows from Proposition~\ref{proposition22}.
Assume that this proposition is true for $n=t$.
Now assume that $|F|=t+1$.
For any $L\in F$, by Proposition~\ref{proposition22}, we have that $K\in S(\mathbf{C}-\{L\})$.
Let $F^{\prime}=F-\{L\}$.
For any $B\in F^{\prime}$, by Proposition~\ref{proposition22}, we have that $B\in S(\mathbf{C}-\{L\})$.
Thus $F^{\prime}\subseteq S(\mathbf{C}-\{L\})$.
By $K\notin F$, we know that $K\notin F^{\prime}$.
Hence $K\in S(\mathbf{C}-\{L\})-F^{\prime}$.
It is obvious $|F^{\prime}|=t$.
By the induction hypothesis, we know that $K\in S((\mathbf{C}-\{L\})-F^{\prime})$.
Since $(\mathbf{C}-\{L\})-F^{\prime}=(\mathbf{C}-\{L\})-(F-\{L\})=\mathbf{C}-F$, $K\in S(\mathbf{C}-F)$.

$(\Leftarrow)$: By Proposition~\ref{proposition36}, we have that $S(\mathbf{C}-F)\subseteq S(\mathbf{C})$. For any $K\in S(\mathbf{C}-F)$, it is obvious $K\notin F$. Therefore $K\in S(\mathbf{C})-F$.
$\Box$

\end{proof}

Below we give the definition of the reduct of a covering, which is somewhat different in form from its definition in ~\cite{ZhuWang03Reduction}.

\begin{definition}(Reduct)
\label{definition10}
Let $\mathbf{C}$ be a covering on $\cup\mathbf{C}$. The reduct of $\mathbf{C}$ is defined by: $reduct(\mathbf{C})=\mathbf{C}-S(\mathbf{C})$.
\end{definition}

$reduct(\mathbf{C})$ has the following property.

\begin{proposition}
\label{proposition25}
$\mathbf{C}\subseteq I(reduct(\mathbf{C}))$.
\end{proposition}

\begin{proof}

For any $K\in\mathbf{C}$, $K\in reduct(\mathbf{C})$ or $K\in S(\mathbf{C})$. If $K\in reduct(\mathbf{C})$, $\{K\}\subseteq reduct(\mathbf{C})$. By $K=\cup\{K\}$, we have that $K\in I(reduct(\mathbf{C}))$. If $K\in S(\mathbf{C})$, $reduct(\mathbf{C})\\
\cup\{K\}=(\mathbf{C}-S(\mathbf{C}))\cup\{K\}=\mathbf{C}-(S(\mathbf{C})-\{K\})$.
By Proposition~\ref{proposition35}, we have that $K\in S(\mathbf{C}-(S(\mathbf{C})-\{K\}))$.
Hence $K\in S(reduct(\mathbf{C})\cup\{K\})$.
By Definition~\ref{definition9}, we know that $K\in I((reduct(\mathbf{C})\cup\{K\})-\{K\})=I(reduct(\mathbf{C}))$. Thus $\mathbf{C}\subseteq I(reduct(\mathbf{C}))$.
$\Box$

\end{proof}

Based on the above proposition, we obtain the following proposition.

\begin{proposition}
\label{proposition26}
Let $\mathbf{C}$ be a covering on $\cup\mathbf{C}$ and $B\subseteq \mathbf{C}$. If $\mathbf{C}\subseteq I(B)$ and for any $K\in B$, $\mathbf{C}\nsubseteq I(B-\{K\})$, $B=reduct(\mathbf{C})$.
\end{proposition}

\begin{proof}
Suppose $reduct(\mathbf{C})-B\neq\emptyset$ and $A\in reduct(\mathbf{C})-B$. It is obvious $A\notin S(\mathbf{C})$. Thus for any $D\subseteq\mathbf{C}-\{A\}$, it follows that $A\neq\cup D$. Since $B\subseteq\mathbf{C}-\{A\}$, $A\neq\cup F$ for any $F\subseteq B$. Hence $A\notin I(B)$. Thus $\mathbf{C}\nsubseteq I(B)$. It is contradictory. Therefore $reduct(\mathbf{C})\subseteq B$. Suppose $B-reduct(\mathbf{C})\neq\emptyset$ and $L\in B-reduct(\mathbf{C})$. It is obvious $reduct(\mathbf{C})\subseteq B-\{L\}$. By Proposition~\ref{proposition25} and Proposition~\ref{proposition24}, we have that $\mathbf{C}\subseteq I(reduct(\mathbf{C}))\subseteq I(B-\{L\})$. It is contradictory. Thus $B-reduct(\mathbf{C})=\emptyset$. Hence $B\subseteq reduct(\mathbf{C})$. Therefore $B=reduct(\mathbf{C})$.
$\Box$

\end{proof}

According to Propositions~\ref{proposition25} and~\ref{proposition26}, we obtain a necessary and sufficient condition for a subset of a set family to be the reduct of the family.

\begin{theorem}
\label{theorem27}
Let $\mathbf{C}$ be a covering on $\cup\mathbf{C}$ and $B\subseteq \mathbf{C}$. $B=reduct(\mathbf{C})$ iff $\mathbf{C}\subseteq I(B)$ and for any $K\in B$, $\mathbf{C}\nsubseteq I(B-\{K\})$.
\end{theorem}

\begin{proof}
$(\Rightarrow)$: By Proposition~\ref{proposition25}, we have that $\mathbf{C}\subseteq I(reduct(\mathbf{C}))=I(B)$. For any $K\in reduct(\mathbf{C})$, we know that $K\notin I(\mathbf{C}-\{K\})$. By $(reduct(\mathbf{C})-\{K\})\subseteq(\mathbf{C}-\{K\})$ and Proposition~\ref{proposition24}, we have that $K\notin I(reduct(\mathbf{C})-\{K\})=I(B-\{K\})$.

$(\Leftarrow)$: It follows from Proposition~\ref{proposition26}.
$\Box$
\end{proof}

Considering both the concepts of reduct and covering of neighborhoods, we have the following proposition.

\begin{proposition}
\label{proposition11}
$reduct(Cov(\mathbf{C}))=Cov(\mathbf{C})$.
\end{proposition}

\begin{proof}
By Definition~\ref{definition10}, we need to prove only $S(Cov(\mathbf{C}))=\emptyset$.
We use the proof by contradiction. Suppose $S(Cov(\mathbf{C}))\neq\emptyset$ and $N(x)\in S(Cov(\mathbf{C}))$, where $x\in\cup\mathbf{C}$. By Definition~\ref{definition9}, we know that there exists some $D\subseteq Cov(\mathbf{C})-\{N(x)\}$ such that $N(x)=\cup D$. Thus $x\in\cup D$. Then there exists some $N(y)\in D$ such that $x\in N(y)$. By Proposition~\ref{proposition3}, we have that $N(x)\subseteq N(y)$. But by $N(y)\in D\subseteq Cov(\mathbf{C})-\{N(x)\}$, we have that $N(y)\subset N(x)$. It is contradictory. $\Box$
\end{proof}

For obtaining a necessary and sufficient condition for a covering of neighborhoods to be a reduct, we propose the following concept.

\begin{definition}
\label{definition13}
Let $\mathbf{C}$ be a covering on $\cup\mathbf{C}$. For any $x\in\cup\mathbf{C}$, we define $\Gamma_{\mathbf{C}}(x)$ by: $\Gamma_{\mathbf{C}}(x)=\{K\in\mathbf{C}:x\in K\wedge\forall y(y\in K\rightarrow\partial_{\mathbf{C}}(\{x,y\})=\partial_{\mathbf{C}}(x))\}$. When there is no confusion, we omit the subscript $\mathbf{C}$.
\end{definition}

To illustrate the above definition, let us see an example.

\begin{example}
\label{example16}
Let $\mathbf{C}=\{\{1,2\},\{1,2,3\},\{3,4\}\}$. Then $\Gamma_{\mathbf{C}}(1)=\Gamma_{\mathbf{C}}(2)=\{\{1,2\}\}$, $\Gamma_{\mathbf{C}}(4)=\{\{3,4\}\}$, $\Gamma_{\mathbf{C}}(3)=\emptyset$ and $\{1,2,3\}\notin\Gamma_{\mathbf{C}}(1)\cup\Gamma_{\mathbf{C}}(2)\cup\Gamma_{\mathbf{C}}(3)\cup\Gamma_{\mathbf{C}}(4)$.
\end{example}

$\Gamma_{\mathbf{C}}(x)$ has the following property.

\begin{proposition}
\label{proposition14}
$|\Gamma_{\mathbf{C}}(x)|\leq1$.
\end{proposition}

\begin{proof}
Let $K_{1}\in\Gamma_{\mathbf{C}}(x)$, $K_{2}\in\Gamma_{\mathbf{C}}(x)$ and $y\in K_{1}$. By Definition~\ref{definition13}, we know that $\partial_{\mathbf{C}}(\{x,y\})=\partial_{\mathbf{C}}(x)$. By Proposition~\ref{proposition7}, we have that $\{K\in\mathbf{C}:x\in K\}=\{K\in\mathbf{C}:\{x,y\}\subseteq K\}$. Since $x\in K_{2}$, $K_{2}\in\{K\in\mathbf{C}:x\in K\}$. Hence $K_{2}\in\{K\in\mathbf{C}:\{x,y\}\subseteq K\}$, then $\{x,y\}\subseteq K_{2}$, thus $y\in K_{2}$. Hence $K_{1}\subseteq K_{2}$. Similarly, $K_{2}\subseteq K_{1}$. Therefore $K_{1}=K_{2}$. Hence $|\Gamma_{\mathbf{C}}(x)|\leq1$. $\Box$
\end{proof}

By the above proposition, we obtain the following corollary.

\begin{corollary}
\label{corollary60}
If $\Gamma_{\mathbf{C}}(x)\neq\emptyset$, $\cup\Gamma_{\mathbf{C}}(x)\in \mathbf{C}$.
\end{corollary}

The following proposition gives a relationship between $\Gamma_{\mathbf{C}}(x)$ and $N_{\mathbf{C}}(x)$.

\begin{proposition}
\label{proposition17}
If $\Gamma_{\mathbf{C}}(x)\neq\emptyset$, $\cup\Gamma_{\mathbf{C}}(x)=N_{\mathbf{C}}(x)$.
\end{proposition}

\begin{proof}
Let $x\in A\in\mathbf{C}$ and $y\in\cup\Gamma_{\mathbf{C}}(x)$. By Definition~\ref{definition13}, we know that $\partial_{\mathbf{C}}(\{x,y\})=\partial_{\mathbf{C}}(x)$. By Proposition~\ref{proposition7}, we have that $\{K\in\mathbf{C}:x\in K\}=\{K\in\mathbf{C}:\{x,y\}\subseteq K\}$. Since $A\in\{K\in\mathbf{C}:x\in K\}$, $A\in\{K\in\mathbf{C}:\{x,y\}\subseteq K\}$. Hence $\{x,y\}\subseteq A$, thus $y\in A$. Therefore $\cup\Gamma_{\mathbf{C}}(x)\subseteq A$. By Definition~\ref{definition13} and Corollary~\ref{corollary60}, we know that $\cup\Gamma_{\mathbf{C}}(x)\in\{K\in\mathbf{C}:x\in K\}$. Therefore $\cup\Gamma_{\mathbf{C}}(x)=\cap\{K\in\mathbf{C}:x\in K\}=N_{\mathbf{C}}(x)$. $\Box$
\end{proof}

Based on the above proposition, we have the following proposition.

\begin{proposition}
\label{proposition18}
$N_{\mathbf{C}}(x)\in\mathbf{C}$ iff $\Gamma_{\mathbf{C}}(x)\neq\emptyset$.
\end{proposition}

\begin{proof}
$(\Rightarrow)$: We use the proof by contradiction. Suppose $\Gamma_{\mathbf{C}}(x)=\emptyset$. For any $K\in\{L\in\mathbf{C}:x\in L\}$, it is obvious there exists some $y\in K$ such that $\partial_{\mathbf{C}}(\{x,y\})\neq\partial_{\mathbf{C}}(x)$. By Proposition~\ref{proposition8}, we know that $y\notin N_{\mathbf{C}}(x)$. Thus $K\neq N_{\mathbf{C}}(x)$. Hence $N_{\mathbf{C}}(x)\notin\{L\in\mathbf{C}:x\in L\}$. Since $x\in N_{\mathbf{C}}(x)$, $N_{\mathbf{C}}(x)\notin\mathbf{C}$. It is a contradiction to the hypothesis.

$(\Leftarrow)$: It follows from Proposition~\ref{proposition17} and Corollary~\ref{corollary60}.
$\Box$
\end{proof}

For obtaining and proving a necessary and sufficient condition for a covering of neighborhoods to be a reduct, we need the following simple property of neighborhoods.

\begin{proposition}
\label{proposition21}
$\mathbf{C}\subseteq I(Cov(\mathbf{C}))$.
\end{proposition}

\begin{proof}
For any $K\in\mathbf{C}$ and any $x\in K$, we have that $\{x\}\subseteq N_{\mathbf{C}}(x)\subseteq K$. Then  $\cup_{x\in K}\{x\}\subseteq\cup_{x\in K}N_{\mathbf{C}}(x)\subseteq\cup_{x\in K}K$. Thus $K\subseteq\cup_{x\in K}N_{\mathbf{C}}(x)\subseteq K$. Hence $K=\cup_{x\in K}N_{\mathbf{C}}(x)$. Therefore $\mathbf{C}\subseteq I(Cov(\mathbf{C}))$. $\Box$
\end{proof}

Based on some above propositions, we obtain a necessary and sufficient condition for a covering of neighborhoods to be a reduct.

\begin{theorem}
\label{theorem19}
$Cov(\mathbf{C})=reduct(\mathbf{C})$ iff for any $x\in\cup\mathbf{C}$, $\Gamma_{\mathbf{C}}(x)\neq\emptyset$.
\end{theorem}

\begin{proof}
$(\Rightarrow)$: We use the proof by contradiction. Suppose there exists some $x\in\cup\mathbf{C}$ such that $\Gamma_{\mathbf{C}}(x)=\emptyset$. By Proposition~\ref{proposition18}, we know that $N_{\mathbf{C}}(x)\notin\mathbf{C}$. Thus $Cov(\mathbf{C})\neq reduct(\mathbf{C})$.

$(\Leftarrow)$: By Proposition~\ref{proposition18}, we know that $Cov(\mathbf{C})\subseteq\mathbf{C}$.
For any $N_{\mathbf{C}}(x)$, by Proposition~\ref{proposition11}, we know that $N_{\mathbf{C}}(x)\notin I(Cov(\mathbf{C})-\{N_{\mathbf{C}}(x)\})$. Since $N_{\mathbf{C}}(x)\in\mathbf{C}$, $\mathbf{C}\nsubseteq I(Cov(\mathbf{C})-\{N_{\mathbf{C}}(x)\})$. Again by Proposition~\ref{proposition21} and Theorem~\ref{theorem27}, we know that $Cov(\mathbf{C})\\
=reduct(\mathbf{C})$. $\Box$
\end{proof}

\section{A condition for two coverings to induce a same relation and a same covering of neighborhoods}
\label{A condition for two coverings to induce a same relation and a same covering of neighborhoods}

In~\cite{XuWang05On}, a binary relation induced by a covering was proposed to establish the relationship between the relation based rough sets and the first type of covering based rough sets. Afterwards, this type of binary relation has been studied further~\cite{YanlanZhangMaokangLuo2013Relationshipsbetweencovering,Zhu09RelationshipBetween}. In this section, we will  give a necessary and sufficient condition for two coverings to induce a same relation. In addition, we will prove that under the same condition two coverings induce a same covering of neighborhoods.

\begin{definition}(Successor neighborhood)
\label{definition28}
Let $R$ be a binary relation on $U$ and $x\in U$. The successor neighborhood of $x$ is defined by: $S_{R}(x)=\{y:xRy\}$.
\end{definition}

Now we introduce the method of inducing a relation by a covering.

\begin{definition}(Relation induced by a covering~\cite{XuWang05On})
\label{definition29}
Let $\mathbf{C}$ be a covering of $\cup\mathbf{C}$. The relation induced by $\mathbf{C}$ is defined by: $R(\mathbf{C})=\{(x,y):
x\in\cup\mathbf{C}\wedge y\in N_{\mathbf{C}}(x)\}$.
\end{definition}

To illustrate the above definition and that two different coverings can induce a same relation, let us see an example.

\begin{example}
\label{example30}
Let $U=\{1,2,3\}$, $\mathbf{C_{1}}=\{\{1,2\},\{2,3\},\{3\}\}$ and $\mathbf{C_{2}}=\{\{1,2,3\},\{1,2\},\\
\{2,3\},\{3\}\}$.
Then $N_{\mathbf{C_{1}}}(1)=N_{\mathbf{C_{2}}}(1)=\{1,2\}$, $N_{\mathbf{C_{1}}}(2)=N_{\mathbf{C_{2}}}(2)=\{2\}$ and $N_{\mathbf{C_{1}}}(3)=N_{\mathbf{C_{2}}}(3)=\{3\}$. Thus $R(\mathbf{C_{1}})=\{(1,1),(1,2),
(2,2),(3,3)\}=R(\mathbf{C_{2}})$.
\end{example}

By Definitions~\ref{definition28} and~\ref{definition29}, we obtain the following simple proposition.

\begin{proposition}
\label{proposition31}
$S_{R(\mathbf{C})}(x)=N_{\mathbf{C}}(x)$.
\end{proposition}

\begin{proof}
$y\in S_{R(\mathbf{C})}(x)\Leftrightarrow(x,y)\in R(\mathbf{C})\Leftrightarrow y\in N_{\mathbf{C}}(x)$.
 $\Box$
\end{proof}

By the above proposition, we obtain the following simple proposition.

\begin{proposition}
\label{proposition32}
Let $\mathbf{C_{1}}$ and $\mathbf{C_{2}}$ be two coverings on $U$. $R(\mathbf{C_{1}})=R(\mathbf{C_{2}})$ iff for any $x\in U$, $N_{\mathbf{C_{1}}}(x)=N_{\mathbf{C_{2}}}(x)$.
\end{proposition}

\begin{proof}
$R(\mathbf{C_{1}})=R(\mathbf{C_{2}})\Leftrightarrow\forall x(x\in U\rightarrow S_{R(\mathbf{C_{1}})}(x)=S_{R(\mathbf{C_{2}})}(x))\Leftrightarrow\forall x(x\in U\rightarrow N_{\mathbf{C_{1}}}(x)=N_{\mathbf{C_{2}}}(x))$. $\Box$
\end{proof}

Below we discuss the condition for the neighborhoods of a same element in different coverings to be equal. First, we propose a definition based on repeat degree as follows.

\begin{definition}
\label{definition33}
Let $\mathbf{C}$ be a covering on $\cup\mathbf{C}$ and $x\in\cup\mathbf{C}$. A mapping $P_{\mathbf{C}}:\cup\mathbf{C}\rightarrow2^{\cup\mathbf{C}}$ is defined by:
$P_{\mathbf{C}}(x)=\{y\in\cup\mathbf{C}:\partial(\{x,y\})=\partial(x)\}$.
\end{definition}

$P_{\mathbf{C}}(x)$ has the following property.

\begin{proposition}
\label{proposition34}
$N_{\mathbf{C}}(x)=P_{\mathbf{C}}(x)$.
\end{proposition}

\begin{proof}
By Proposition~\ref{proposition8} and Definitions~\ref{definition33}, we have that $y\in N_{\mathbf{C}}(x)\Leftrightarrow\partial(\{x,y\})=\partial(x)\Leftrightarrow y\in P_{\mathbf{C}}(x)$. $\Box$
\end{proof}

By this proposition, we obtain a necessary and sufficient condition for two coverings to induce a some relation.

\begin{theorem}
\label{theorem37}
Let $\mathbf{C_{1}}$ and $\mathbf{C_{2}}$ be two coverings on $U$. $R(\mathbf{C_{1}})=R(\mathbf{C_{2}})$ iff for any $x\in U$, $P_{\mathbf{C_{1}}}(x)=P_{\mathbf{C_{2}}}(x)$.
\end{theorem}

\begin{proof}
By Propositions~\ref{proposition32} and ~\ref{proposition34}, we have that $R(\mathbf{C_{1}})=R(\mathbf{C_{2}})\Leftrightarrow\forall x(x\in U\rightarrow N_{\mathbf{C_{1}}}(x)=N_{\mathbf{C_{2}}}(x))\Leftrightarrow\forall x(x\in U\rightarrow P_{\mathbf{C_{1}}}(x)=P_{\mathbf{C_{2}}}(x))$. $\Box$
\end{proof}

In Example~\ref{example30}, we see that $\mathbf{C_{1}}\neq\mathbf{C_{2}}$ but $Cov(\mathbf{C_{1}})=Cov(\mathbf{C_{2}})$. Below we discuss the condition for two coverings to induce a same covering of neighborhoods. First, we have the following proposition.

\begin{proposition}
\label{proposition38}
Let $\mathbf{C_{1}}$ and $\mathbf{C_{2}}$ be two coverings on $U$. $Cov(\mathbf{C_{1}})=Cov(\mathbf{C_{2}})$ iff for any $x\in U$, $N_{\mathbf{C_{1}}}(x)=N_{\mathbf{C_{2}}}(x)$.
\end{proposition}

\begin{proof}
$(\Rightarrow)$: Let
$Cov(\mathbf{C}_{1})=Cov(\mathbf{C}_{2})=\{K_{1},K_{2},\cdots,K_{t}\}$. We use the proof by contradiction. Suppose there exist some $1\leq i<j\leq t$ and $x\in U$ such that $N_{\mathbf{C}_{1}}(x)=K_{i}$ and $N_{\mathbf{C}_{2}}(x)=K_{j}$, where $K_{i}\neq K_{j}$. It is obvious $x\in K_{i}$ and $x\in K_{j}$. By $N_{\mathbf{C}_{2}}(x)\neq K_{i}$, we know that there exists some $w\in U-\{x\}$ such that $N_{\mathbf{C}_{2}}(w)=K_{i}$. By Proposition~\ref{proposition3}, we have that $N_{\mathbf{C}_{2}}(x)\subset N_{\mathbf{C}_{2}}(w)$. Thus $K_{j}\subset K_{i}$.
By $N_{\mathbf{C}_{1}}(x)\neq K_{j}$, we know that there exists some $y\in U-\{x\}$ such that $N_{\mathbf{C}_{1}}(y)=K_{j}$. By Proposition~\ref{proposition3}, we have that $N_{\mathbf{C}_{1}}(x)\subset N_{\mathbf{C}_{1}}(y)$. Thus $K_{i}\subset K_{j}$. It is contradictory.

$(\Leftarrow)$: It is straightforward. $\Box$
\end{proof}

Based some above propositions, we obtain a condition for different coverings to induce a same relation and a same covering of neighborhoods.

\begin{theorem}
\label{theorem39}
Let $\mathbf{C_{1}}$ and $\mathbf{C_{2}}$ be two coverings on $U$. Then the following statements are equivalent:\\
$(1)$ For any $x\in U$, $P_{\mathbf{C_{1}}}(x)=P_{\mathbf{C_{2}}}(x)$,\\
$(2)$ $R(\mathbf{C_{1}})=R(\mathbf{C_{2}})$,\\
$(3)$ $Cov(\mathbf{C_{1}})=Cov(\mathbf{C_{2}})$.
\end{theorem}

\begin{proof}
$(1)\Leftrightarrow(2)$: It has been proved in Theorem~\ref{theorem37}.

$(1)\Leftrightarrow(3)$: It follows from Proposition~\ref{proposition34} and Proposition~\ref{proposition38}. $\Box$
\end{proof}

In the end of this section, we give the following theorem.

\begin{theorem}
\label{theorem40}
$Cov(\mathbf{C})=\{P_{\mathbf{C}}(x):x\in\cup\mathbf{C}\}$.
\end{theorem}

\begin{proof}
It follows from Definition~\ref{definition4} and Proposition~\ref{proposition34}. $\Box$
\end{proof}

The difference between defining $Cov(\mathbf{C})$ by $\{N(x):x\in\cup\mathbf{C}\}$ and defining $Cov(\mathbf{C})$ by $\{P_{\mathbf{C}}(x):x\in\cup\mathbf{C}\}$ will be clear after we see the difficulty of calculating the covering by repeat degree of partial subsets in the following section.

\section{Calculating the covering by repeat degree}
\label{Calculating the covering by repeat degree}
Given a set family $\mathbf{C}$, we can calculate the repeat degree of any subset of $\cup\mathbf{C}$. Conversely, can we calculate the covering by repeat degree of all even partial subsets? In this section, we will discuss this issue.

\begin{definition}
\label{definition41}
Let $\mathbf{C}$ be a covering on $\cup\mathbf{C}$ and $X\subseteq\cup\mathbf{C}$.
We use $\rho_{\mathbf{C}}(X)=1$ and $\rho_{\mathbf{C}}(X)=0$ to express that $X\in\mathbf{C}$ and $X\notin\mathbf{C}$, respectively. When there is no confusion, we omit the subscript $\mathbf{C}$.
\end{definition}

For the convenience of writing, we denote $\rho_{\mathbf{C}}(\{x\})$ as $\rho_{\mathbf{C}}(x)$. To illustrate this definition, let us see an example.

\begin{example}
\label{example42}
Let $\mathbf{C}=\{\{1,2\},\{2,3\}\}$.
Then $\rho(\emptyset)=\rho(1)=\rho(2)=\rho(3)=\rho(\{1,3\})=\rho(\{1,2,3\})=0$ and $\rho(\{1,2\})=\rho(\{2,3\})=1$.
\end{example}

By Definition~\ref{definition5} and Definition~\ref{definition41}, the following proposition holds obviously.

\begin{proposition}
\label{proposition43}
$\partial_{\mathbf{C}}(Y)=\Sigma_{Y\subseteq X\subseteq\cup\mathbf{C}}\rho_{\mathbf{C}}(X)$.
\end{proposition}

Particularly, we have that $\partial_{\mathbf{C}}(\cup\mathbf{C})=\rho_{\mathbf{C}}(\cup\mathbf{C})$.

\begin{definition}
\label{definition46}
Let $\mathbf{C}$ be a covering on $\cup\mathbf{C}$. We define $\delta(\mathbf{C})=\{(X,\rho_{\mathbf{C}}(X)):X\subseteq\cup\mathbf{C}\wedge X\neq\emptyset\}$.
\end{definition}

To illustrate this definition, let us see an example.

\begin{example}
\label{example47}
Let $\mathbf{C}=\{\{1,2\},\{2,3\}\}$. Then $\delta(\mathbf{C})=\{(\{1\},0),(\{2\},0),
(\{3\},0),
(\{1,2\\
\},1),(\{2,3\},1),(\{1,3\},0),(\{1,2,3\},0)\}$.
\end{example}

The following two propositions indicate that $\delta$ is a bijection.

\begin{proposition}
\label{proposition48}
If $\mathbf{C_{1}}\neq\mathbf{C_{2}}$, $\delta(\mathbf{C_{1}})\neq\delta(\mathbf{C_{2}})$.
\end{proposition}

\begin{proof}
It is obvious $\mathbf{C_{1}}-\mathbf{C_{2}}\neq\emptyset$ or $\mathbf{C_{2}}-\mathbf{C_{1}}\neq\emptyset$. Without loss of generality, suppose $\mathbf{C_{1}}-\mathbf{C_{2}}\neq\emptyset$ and $K\in\mathbf{C_{1}}-\mathbf{C_{2}}$. Then $(K,1)\in\delta(\mathbf{C_{1}})-\delta(\mathbf{C_{2}})$. Hence $\delta(\mathbf{C_{1}})\neq\delta(\mathbf{C_{2}})$. $\Box$
\end{proof}

\begin{proposition}
\label{proposition49}
Let $U$ be a finite and nonempty set. For any $\{(X,f(X)):X\subseteq U\wedge X\neq\emptyset\}$, where $f(X)\in\{0,1\}$, there exists a set family $\mathbf{C}$ such that $\{(X,f(X)):X\subseteq U\wedge X\neq\emptyset\}=\delta(\mathbf{C})$.
\end{proposition}

\begin{proof}
Let $\mathbf{C}=\{K\subseteq U:X\neq\emptyset\wedge f(K)=1\}$. Then $\{(X,f(X)):X\subseteq U\}=\delta(\mathbf{C})$. $\Box$
\end{proof}

We give a new concept as follows.

\begin{definition}
\label{definition51}
Let $\mathbf{C}$ be a covering on $\cup\mathbf{C}$, $|\cup\mathbf{C}|=n$ and $W\subseteq\{0,1,\cdots,n\}$. We define $D_{\mathbf{C}}(W)=\{(X,\partial_{\mathbf{C}}(X)):
X\subseteq\cup\mathbf{C}\wedge|X|\in W\}$. When there is no confusion, we omit the subscript $\mathbf{C}$. Particularly, $D_{\mathbf{C}}(\{1,2,\cdots,n\})$ is written as $D(\mathbf{C})$ for short.
\end{definition}

By the above definition, we have the following proposition.

\begin{proposition}
\label{proposition52}
Let $\mathbf{C_{1}}$ be a covering on $\cup\mathbf{C_{1}}$ and $\mathbf{C_{2}}$ be a covering on $\cup\mathbf{C_{2}}$. If $\mathbf{C_{1}}\neq\mathbf{C_{2}}$, $D(\mathbf{C_{1}})\neq D(\mathbf{C_{2}})$.
\end{proposition}

\begin{proof}
We use the proof by contradiction. Suppose $D(\mathbf{C_{1}})=D(\mathbf{C_{2}})$. Then $\cup\mathbf{C_{1}}=\cup\mathbf{C_{2}}$. Let $U=\cup\mathbf{C_{1}}=\cup\mathbf{C_{2}}$ and $|U|=n$. For any $K\subseteq U$, we claim that $\rho_{\mathbf{C_{1}}}(K)=\rho_{\mathbf{C_{2}}}(K)$. We prove this assertion using induction on $n-|K|$. If $n-|K|=0$, $K=U$. Thus $\rho_{\mathbf{C_{1}}}(U)=\partial_{\mathbf{C_{1}}}(U)=\partial_{\mathbf{C_{2}}}(U)=\rho_{\mathbf{C_{2}}}(U)$. Assume this assertion is true
for $n-|K|\leq t-1$. Now assume $n-|K|=t$. By Proposition~\ref{proposition43}, we have that $\rho_{\mathbf{C_{1}}}(K)=\partial_{\mathbf{C_{1}}}(K)-\Sigma_{K\subset X\subseteq\cup\mathbf{C}}\rho_{\mathbf{C_{1}}}(X)=\partial_{\mathbf{C_{2}}}(K)-\Sigma_{K\subset X\subseteq\cup\mathbf{C}}\rho_{\mathbf{C_{2}}}(X)=\rho_{\mathbf{C_{2}}}(K)$. Hence $\delta(\mathbf{C_{1}})=\delta(\mathbf{C_{2}})$. By Proposition~\ref{proposition48}, we have that $\mathbf{C_{1}}=\mathbf{C_{2}}$. It is contradictory. $\Box$
\end{proof}

Below we give a method of calculating the covering by repeat degree. In fact, we have the following theorem.

\begin{theorem}
\label{theorem55}
Let $U$ be a finite and nonempty set. Let for any $Y\subseteq U$, mappings $f:2^{U}\rightarrow R$ and $g:2^{U}\rightarrow R$ satisfy $f(Y)=\Sigma_{Y\subseteq X\subseteq U}g(X)$. Then for any $V\subseteq U$, $g(V)=\Sigma_{V\subseteq Z\subseteq U}((-1)^{|Z|-|V|}f(Z))$.
\end{theorem}

\begin{proof}
Let $|U|-|V|=k$. We prove this assertion using induction on $k$. By $f(Y)=\Sigma_{Y\subseteq X\subseteq U}g(X)$, we have that $f(U)=g(U)$. Thus this assertion is true for $k=0$.
Assume this assertion is true for $k\leq t-1$. Now assume $k=t$. Let $Z\subseteq U$ such that $V\subset Z$. Let $V(i)=\{X\subseteq U|V\subset X\wedge|X|-|V|=i\}$ and $Z(i)=\{H\in V(i)|H\subseteq Z\}$. We have that $|Z(i)|=C_{|Z|-|V|}^{i}$.
For any $V\subset X\subseteq U$, it is obvious $|U|-|X|\leq t-1$. By the assumption of the induction and $\Sigma_{i=0}^{n}(-1)^{i}C_{n}^{i}=0$, we have that
\begin{center}
$f(V)=\Sigma_{V\subseteq X\subseteq U}g(X)~~~~~~~~~~~~~~~~~~~~~~~~~~~~~~~~~~~~~~~~~~~~~~~~$
\end{center}
\begin{center}
$=g(V)+\Sigma_{V\subset X\subseteq U}g(X)~~~~~~~~~~~~~~~~~~~~~~~~~$
\end{center}
\begin{center}
$~~~~~~~~~=g(V)+\Sigma_{V\subset X\subseteq U}\Sigma_{X\subseteq Z\subseteq U}(-1)^{|Z|-|X|}f(Z)$
\end{center}
\begin{center}
$~~~~~~~~~~~~~~~~~~~~~=g(V)+\Sigma_{i=1}^{k}(\Sigma_{V\subset Z\subseteq U}(-1)^{|Z|-|V|-i}C_{|Z|-|V|}^{i}f(Z))$
\end{center}
\begin{center}
$~~~~~~~~~~~~~~~~~~~~~=g(V)+\Sigma_{V\subset Z\subseteq U}(f(Z)\Sigma_{i=1}^{k}(-1)^{|Z|-|V|-i}C_{|Z|-|V|}^{i})$
\end{center}
\begin{center}
$=g(V)+\Sigma_{V\subset Z\subseteq U}((-1)^{|Z|-|V|-1}f(Z))$.
\end{center}
Thus
\begin{center}
$~~~~~~g(V)=f(V)-\Sigma_{V\subset Z\subseteq U}((-1)^{|Z|-|V|-1}f(Z))~~~~~~~~~~~~~~~$
\end{center}
\begin{center}
$~~~~~~~~~~=f(V)+\Sigma_{V\subset Z\subseteq U}((-1)^{|Z|-|V|}f(Z))~~~~~~~~~~~~~~~$
\end{center}
\begin{center}
$=\Sigma_{V\subseteq Z\subseteq U}(-1)^{|Z|-|V|}f(Z).~\Box~~~~~~~~~~~~~~$
\end{center}

\end{proof}

By the above theorem and Proposition~\ref{proposition43}, we have the following corollary.

\begin{corollary}
\label{corollary56}
$\rho_{\mathbf{C}}(V)=\Sigma_{V\subseteq Z\subseteq U}((-1)^{|Z|-|V|}\partial_{\mathbf{C}}(Z))$.
\end{corollary}

Given a $D(\mathbf{C})$, by the above corollary, we can calculate the $\delta(\mathbf{C})$. Again by Proposition~\ref{proposition49}, we obtain $\mathbf{C}$. On the other hand, let $\mathbf{C_{1}}$ and $\mathbf{C_{2}}$ be two coverings on $U$ and $|U|=n$. If $W\subset\{1,2,\cdots,n\}$, $\mathbf{C_{1}}\neq\mathbf{C_{2}}$ does not imply $D_{\mathbf{C_{1}}}(W)\neq D_{\mathbf{C_{1}}}(W)$. To illustrate this, let us see an example.

\begin{example}
\label{example53}
Let $\mathbf{C_{1}}=\{\{a,b\},\{b,c\},\{a,c\}\}$ and $\mathbf{C_{2}}=\{\{a,b,c\},\{a\},\{b\},\{c\}\}$. Then $D_{\mathbf{C_{1}}}(\{1,2\})=\{(\{a\},2),(\{b\},2),(\{c\},2),(\{a,b\},1),(\{b,c\},1),
(\{a,c\},1)\}=D_{\mathbf{C_{2}}}(\{1,2\})$.
\end{example}

In fact, we have the following proposition, in which we denote $|X|$ as $Card(X)$.

\begin{proposition}
\label{proposition54}
Let $\mathbf{C_{1}}$ and $\mathbf{C_{2}}$ be two coverings on $U$ and $Card(U)=n>1$. $\mathbf{C_{1}}\neq\mathbf{C_{2}}$ and $D_{\mathbf{C_{1}}}(\{1,2,\cdots,n-1\})=D_{\mathbf{C_{2}}}(\{1,2,\cdots,n-1\})$ iff $\{\mathbf{C_{1}},\mathbf{C_{2}}\}=\{\{X\subseteq U:2|Card(X)\},\{X\subseteq U:2\nmid Card(X)\}\}$.
\end{proposition}

\begin{proof}
$(\Rightarrow)$: Suppose $\partial_{\mathbf{C_{1}}}(U)=\partial_{\mathbf{C_{2}}}(U)$. Then $D(\mathbf{C_{1}})=\{(U,\partial_{\mathbf{C_{1}}}(U))\}\cup D_{\mathbf{C_{1}}}(\{1,2,\\
\cdots,n-1\})=\{(U,\partial_{\mathbf{C_{2}}}(U))\}\cup D_{\mathbf{C_{2}}}(\{1,2,\cdots,n-1\})=D(\mathbf{C_{2}})$. By Proposition~\ref{proposition52}, we have that $\mathbf{C_{1}}=\mathbf{C_{2}}$. It is contradictory. Thus $\partial_{\mathbf{C_{1}}}(U)\neq\partial_{\mathbf{C_{2}}}(U)$. Without loss of generality, suppose $\partial_{\mathbf{C_{1}}}(U)=0$ and $\partial_{\mathbf{C_{2}}}(U)=1$. It is obvious $\rho_{\mathbf{C_{1}}}(U)=0$ and $\rho_{\mathbf{C_{2}}}(U)=1$. For any $X\subseteq U$, let $Card(U)-Card(X)=k$. We claim that $X\in\mathbf{C_{1}}$ iff $2\nmid k$ and $X\in\mathbf{C_{2}}$ iff $2|k$. We prove this assertion using induction on $k$. If $k=1$, by Proposition~\ref{proposition43}, we have that $\partial_{\mathbf{C_{1}}}(X)=\rho_{\mathbf{C_{1}}}(X)+\rho_{\mathbf{C_{1}}}(U)\leq1$ and $\partial_{\mathbf{C_{2}}}(X)=\rho_{\mathbf{C_{2}}}(X)+\rho_{\mathbf{C_{2}}}(U)\geq1$. Thus $\partial_{\mathbf{C_{1}}}(X)=\partial_{\mathbf{C_{2}}}(X)=1$, $\rho_{\mathbf{C_{1}}}(X)=1$ and $\rho_{\mathbf{C_{2}}}(X)=0$. Hence $X\in\mathbf{C_{1}}$ and $X\notin\mathbf{C_{2}}$. If $k=2$, by Proposition~\ref{proposition43}, we have that $\partial_{\mathbf{C_{1}}}(X)=\rho_{\mathbf{C_{1}}}(X)+\Sigma_{X\subset Y\subseteq\cup\mathbf{C}}\rho_{\mathbf{C_{1}}}(Y)=\rho_{\mathbf{C_{1}}}(X)+C_{2}^{1}=\rho_{\mathbf{C_{1}}}(X)+2\geq2$ and $\partial_{\mathbf{C_{2}}}(X)=\rho_{\mathbf{C_{2}}}(X)+\rho_{\mathbf{C_{2}}}(U)=\rho_{\mathbf{C_{2}}}(X)+1\leq2$.
Thus $\partial_{\mathbf{C_{1}}}(X)=\partial_{\mathbf{C_{2}}}(X)=2$, $\rho_{\mathbf{C_{1}}}(X)=0$ and $\rho_{\mathbf{C_{2}}}(X)=1$. Hence $X\notin\mathbf{C_{1}}$ and $X\in\mathbf{C_{2}}$.
Assume this assertion is true for $k\leq t-1$. Now assume $k=t$. It is obvious $2|t$ or $2\nmid t$. If $2|t$, by Proposition~\ref{proposition43} and the assumption of the induction, we have that $\partial_{\mathbf{C_{1}}}(X)=\rho_{\mathbf{C_{1}}}(X)+\Sigma_{X\subset Y\subseteq\cup\mathbf{C}}\rho_{\mathbf{C_{1}}}(Y)=\rho_{\mathbf{C_{1}}}(X)+\Sigma_{i=1}^{\frac{t}{2}}C_{t}^{2i-1}$ and $\partial_{\mathbf{C_{2}}}(X)=\rho_{\mathbf{C_{2}}}(X)+\Sigma_{X\subset Y\subseteq\cup\mathbf{C}}\rho_{\mathbf{C_{2}}}(Y)=\rho_{\mathbf{C_{2}}}(X)+\Sigma_{i=1}^{\frac{t}{2}}C_{t}^{2i}$. By $\Sigma_{i=1}^{\frac{t}{2}}C_{t}^{2i}=\Sigma_{i=1}^{\frac{t}{2}}C_{t}^{2i-1}-1$ and $\partial_{\mathbf{C_{1}}}(X)=\partial_{\mathbf{C_{2}}}(X)$, we have that $\rho_{\mathbf{C_{2}}}(X)-\rho_{\mathbf{C_{1}}}(X)=1$. Hence $\rho_{\mathbf{C_{1}}}(X)=0$ and $\rho_{\mathbf{C_{2}}}(X)=1$. Therefore $X\notin\mathbf{C_{1}}$ and $X\in\mathbf{C_{2}}$.
If $2\nmid t$, by Proposition~\ref{proposition43} and the assumption of the induction, we have that $\partial_{\mathbf{C_{1}}}(X)=\rho_{\mathbf{C_{1}}}(X)+\Sigma_{X\subset Y\subseteq\cup\mathbf{C}}\rho_{\mathbf{C_{1}}}(Y)=\rho_{\mathbf{C_{1}}}(X)+\Sigma_{i=1}^{\frac{t-1}{2}}C_{t}^{2i}$ and $\partial_{\mathbf{C_{2}}}(X)=\rho_{\mathbf{C_{2}}}(X)+\Sigma_{X\subset Y\subseteq\cup\mathbf{C}}\rho_{\mathbf{C_{2}}}(Y)=\rho_{\mathbf{C_{2}}}(X)+\Sigma_{i=1}^{\frac{t+1}{2}}C_{t}^{2i-1}$. By $\Sigma_{i=1}^{\frac{t-1}{2}}C_{t}^{2i}=\Sigma_{i=1}^{\frac{t+1}{2}}C_{t}^{2i-1}-1$ and $\partial_{\mathbf{C_{1}}}(X)=\partial_{\mathbf{C_{2}}}(X)$, we have that $\rho_{\mathbf{C_{1}}}(X)-\rho_{\mathbf{C_{2}}}(X)=1$. Hence $\rho_{\mathbf{C_{1}}}(X)=1$ and $\rho_{\mathbf{C_{2}}}(X)=0$. Therefore $X\in\mathbf{C_{1}}$ and $X\notin\mathbf{C_{2}}$.

$(\Leftarrow)$: Without loss of generality, suppose $\mathbf{C_{1}}=\{X\subseteq U:2\nmid Card(X)\}$ and $\mathbf{C_{2}}=\{X\subseteq U:2|Card(X)\}$. It is obvious $\mathbf{C_{1}}\neq\mathbf{C_{2}}$. Let $X\subset U$ and $|X|=t$. It is obvious $2|t$ or $2\nmid t$. If $2|t$, $\partial_{\mathbf{C_{1}}}(X)=\Sigma_{i=0}^{[\frac{n-t-1}{2}]}C_{n-t}^{2i+1}$ and $\partial_{\mathbf{C_{2}}}(X)=\Sigma_{i=0}^{[\frac{n-t}{2}]}C_{n-t}^{2i}$. If $2\nmid t$, $\partial_{\mathbf{C_{1}}}(X)=\Sigma_{i=0}^{[\frac{n-t}{2}]}C_{n-t}^{2i}$ and $\partial_{\mathbf{C_{2}}}(X)=\Sigma_{i=0}^{[\frac{n-t-1}{2}]}C_{n-t}^{2i+1}$. Since $\Sigma_{i=0}^{[\frac{n-t-1}{2}]}C_{n-t}^{2i+1}=\Sigma_{i=0}^{[\frac{n-t}{2}]}C_{n-t}^{2i}$, $\partial_{\mathbf{C_{1}}}(X)=\partial_{\mathbf{C_{2}}}(X)$. Thus $D_{\mathbf{C_{1}}}(\{1,2,\cdots,n-1\})=D_{\mathbf{C_{1}}}(\{1,2,\cdots,n-1\})$.
$\Box$
\end{proof}

By the above proposition, we know that partial information of repeat degree cannot determine the covering.
In order to calculate the covering by repeat degree, we have to know $D(\mathbf{C})$. It is obvious $P_{\mathbf{C}}(x)$ depends on only $D_{\mathbf{C}}(\{1,2\})$.  Let $\mathbf{C_{1}}$ and $\mathbf{C_{2}}$ be two coverings on $U$. By the above proposition, we know that for any $x\in U$, $P_{\mathbf{C_{1}}}(x)=P_{\mathbf{C_{2}}}(x)$ does not imply $\mathbf{C_{1}}=\mathbf{C_{2}}$. Thus if we know nothing about $\mathbf{C}$ but $D_{\mathbf{C}}(\{1,2\})$, we cannot calculate $\mathbf{C}$. However, we can still calculate $Cov(\mathbf{C})$ by Theorem~\ref{theorem40}. And in this case, with the help of repeat degree, some issues, such as whether $Cov(\mathbf{C})$ is a reduct, whether $Cov(\mathbf{C})$ is a partition, can also be determined.

\section{Conclusions }
\label{S:Conclusions}
In this paper, we studied further the applications of repeat degree on coverings of neighborhoods. We first gave a sufficient and necessary condition for a covering of neighborhoods to be the reduct of the covering inducing it. Then we gave a sufficient and necessary condition for two coverings induce a same relation and a same covering of neighborhoods. Finally, the method of calculating the covering by repeat degree is given. This paper shows that repeat degree plays an important role in the study of coverings of neighborhoods.

\section*{Acknowledgments}
This work is supported in part by the National Natural Science Foundation of China under Grant No. 61170128, the Natural Science Foundation of Fujian Province, China, under Grant No. 2012J01294, and the Science and Technology Key Project of Fujian Province, China, under Grant No. 2012H0043.


\end{document}